%% file: main.tex
\documentclass{article}

\usepackage[final,nonatbib]{neurips_2019}
\usepackage[normalem]{ulem}
\usepackage{xcolor}

\usepackage[utf8]{inputenc} 
\usepackage[T1]{fontenc}    
\usepackage{hyperref}       
\usepackage{url}            
\usepackage{booktabs}       
\usepackage{amsfonts}       
\usepackage{nicefrac}       
\usepackage{microtype}      
\usepackage{amsmath}
\usepackage{amsthm}
\usepackage{graphicx}
\usepackage{caption}

\newcommand{\R}{\mathbb{R}}
\newcommand{\cN}{\mathcal{N}}

\newcommand{\cA}{\mathcal{A}}

\newcommand{\cL}{\mathcal{L}}

\newcommand{\Ex}{\mathbb{E}}

\newcommand{\norm}[1]{\lVert #1 \rVert}

\newtheorem{theorem}{Theorem}
\newtheorem*{theorem*}{Theorem}

\newcommand{\X}{\mathbf{X}}

\newcommand{\x}{\mathbf{x}}

\newcommand{\z}{\mathbf{z}}

\newcommand{\zi}{\z^{(i)}}

\newcommand{\bdelta}{\boldsymbol{\delta}}

\title{Invert and Defend: Model-based Approximate Inversion of Generative Adversarial Networks for Secure Inference}

\author{%
  Wei-An Lin\thanks{Equal contribution.} \\
  University of Maryland \\
  \texttt{walin@umd.edu} \\
  \And
  Yogesh Balaji$^*$ \\
  University of Maryland \\
  \texttt{yogesh@cs.umd.edu} \\
  \AND
  Pouya Samangouei\\
  Butterfly Network \\
  \texttt{pouya@butterflynetwork.com} \\
  \And
  Rama Chellappa \\
  University of Maryland \\
  \texttt{rama@umiacs.umd.edu} \\
}

\begin{document}

\maketitle

\begin{abstract}
    Inferring the latent variable generating a given test sample is a challenging problem in Generative Adversarial Networks (GANs). In this paper, we propose InvGAN - a novel framework for solving the inference problem in GANs, which involves training an encoder network capable of inverting a pre-trained generator network without access to any training data. Under mild assumptions, we theoretically show that using InvGAN, we can approximately invert the generations of any latent code of a trained GAN model. Furthermore, we empirically demonstrate the superiority of our inference scheme by quantitative and qualitative comparisons with other methods that perform a similar task. We also show the effectiveness of our framework in the problem of adversarial defenses where InvGAN can successfully be used as a projection-based defense mechanism. Additionally, we show how InvGAN can be used to implement reparameterization white-box attacks on projection-based defense mechanisms. Experimental validation on several benchmark datasets demonstrate the efficacy of our method in achieving improved performance on several white-box and black-box attacks. Our code is available at \url{https://github.com/yogeshbalaji/InvGAN}.

\end{abstract}

\input{intro.tex}

\input{background.tex}

\input{method.tex}
\input{experiments.tex}
\input{conclusion.tex}


\input{main.bbl}
\input{supp.tex}
\end{document}

%% file: intro.tex
\section{Introduction}\label{sec:intro}

Generative Adversarial Networks (GANs) is one of the most successful frameworks used for generative modeling. Significant research progress in GANs over the last few years has pushed boundaries in generation capabilities and has made possible the synthesis of photo-realistic images of human faces~\cite{karras2019StyleGAN,karras2018progressive} and objects~\cite{BigGAN_2018}. A fundamental problem involving GANs is the problem of inversion -- given a test image, what is the most likely latent code that generates the test sample? The inversion problem is extremely challenging since the generator network in GANs is highly non-linear and non-injective. Inversion has applications in several machine learning problems e.g. domain adaptation~\cite{Bousmalis_2017_CVPR,Hoffman_cycada2017}, compressed sensing \cite{bora2017compressed}, adversarial defenses~\cite{samangouei2018defense}, and anomaly detection~\cite{Schlegl2017AnoGAN}.

In this work, we propose a novel approach for addressing the inversion problem in GANs. Our approach is model-based where the mapping from image space to latent space is represented as a parametric function. We solve for the parameters of this function by sampling the latent codes from the noise distribution of the GAN and making sure that (a) the inversions produced from the generated samples are close to the sampled codes (b) the generated images of the inversions semantically match the GAN generations  and (c) the distribution of inverted images follows the distribution of the GAN. Our method is a data-free inversion mechanism i.e., given a pre-trained generator network, no access to input dataset is needed. This is particularly important in privacy-preserving learning scenarios in which the data provider does not intend to publicly release the data due to privacy reasons, but instead releases a GAN model trained on this data satisfying several privacy constraints~\cite{Xie2018DPGAN}. Our approach can invert such a GAN model.

In addition to comparing the reconstruction performance with previously proposed encoder-GAN models, we demonstrate the effectiveness of our inversion approach for the problem of adversarial defenses. The vulnerability of deep neural networks to small imperceptible perturbations has been demonstrated in several recent papers~\cite{szegedy2013intriguing, goodfellow2014explaining, carlini2017towards, Moosavi2017universal}, and this poses a huge threat in security-critical application domains where these networks are used. 

\begin{figure}\label{fig:overview}
   \centering
   \includegraphics[width=0.95\linewidth]{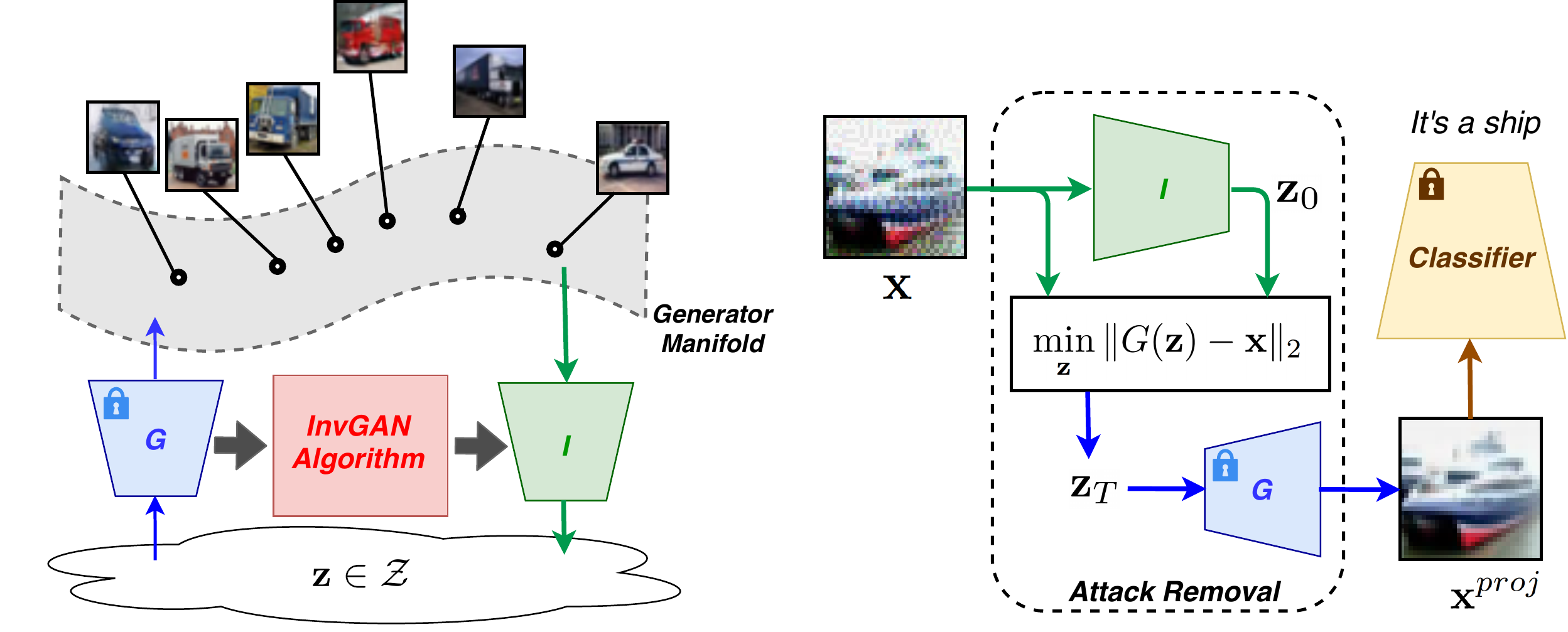}
   \caption{Overview of the proposed InvGAN framework. Left: Given a pre-trained generator $G$ and no data, we solve for $I$ to approximately invert the generator. Right: The application of InvGAN in adversarial defenses, where InvGAN can be used to project an adversarially perturbed sample $\x$ onto the generator manifold, and the projected sample $\x^{proj}$ can be used to make a robust prediction.}
\label{fig:onecol}
\end{figure}

    One of the successful defense strategies proposed recently is DefenseGAN~\cite{samangouei2018defense}, which uses a GAN as a defense mechanism. In DefenseGAN, the inference step is used to project a test sample onto the GAN's manifold to remove possible attacks. The inversion is posed as a non-convex optimization problem which is solved using Gradient Descent. However, this needs careful hyper-parameter tuning per dataset and generator, is extremely slow in practice, and does not scale well for deeper generator models. In this work, we propose an inference procedure to speed up DefenseGAN. More specifically, we use our trained encoder network to initialize the optimization problem. Using this initialization, the inference problem can be solved in very few iterations while preserving the quality of reconstructions. This leads to effortless hyperparameter tuning, a dramatic speed up in runtime, and improved reconstruction results.

Finally, we demonstrate how InvGAN can be used to create adversarial attacks for projection-based defense mechanisms that do not allow gradient computation for crafting white-box attacks. The idea, called ``reparameterization" \cite{athalye2018obfuscated}, is to approximate the projection operation using a differentiable function and derive gradient-based attacks using this approximation. We investigate how InvGAN can be used to implement this attack and analyze its performance in this context. 


In summary, the main contributions of this work are as follows:
\begin{itemize}
    \item [1] A model-based and data-free approach for approximately inverting pre-trained generator networks.
    \item[2] Significant improvements to Defense-GAN yielding improved defense performance, runtime speed up, and attack detection.
    \item [3] An implementation of reparameterization white-box attack \cite{athalye2018obfuscated} and analysis of its effectiveness across different datasets and methods.
\end{itemize}


%% file: background.tex
\section{Background}\label{sec:background}


\subsection{Inverting generative models}

While significant research has focused on improving the quality, stability and diversity of GANs~\cite{BigGAN_2018,karras2019StyleGAN,karras2018progressive}, there has been relatively less work on the inversion problem despite its practical significance. The method in \cite{lipton2017precise} poses inversion as a non-convex optimization problem, which is then solved using projected gradient descent with stochastic clipping. A similar optimization with logistic loss has been proposed in \cite{creswell2018inverting}. While the above two methods are model-free and work for a pre-trained generator, they are extremely slow and deliver poor reconstructions for harder datasets like CIFAR-10.

Another line of work involves modifying the GAN objective to support generation and inference in a unified framework. They typically involve training an encoder function that maps from the image space to the latent space jointly with the generator function in GANs. Methods presented in~\cite{dumoulin2017ALI,donahue2016bigan} trains the encoder network using an adversarial loss on (latent, image) pairs. The method proposed in \cite{Ulyanov2018ALI} uses adversarial and reconstruction loss in latent space to train the encoder. While these methods enable fast inference, modifying the GAN objective to support inference affects the quality of generator models. Our approach offers the best of both worlds -- fast inference and ability to perform inference on a pre-trained generator, thus preserving the quality of generative models.



\subsection{Adversarial attacks and defenses}

Adversarial attacks are imperceptible changes crafted to the input samples by an adversary to flip a model's prediction $\hat{y} = f(\x)$. In this work, we focus on classification problems, hence, $f(\x) \in \{1, \cdots ,c \}$. The most common form of adversarial attacks are additive perturbations where a norm-bounded perturbation $\bdelta$ is added to the input sample $\X \in \R^{w \times h \times c}$ as $\tilde{\X} = \X + \bdelta$. The strength of the attack correlates inversely with the norm-bound of the perturbation -- stronger the attack, lower is the bound~\cite{goodfellow2014explaining}. Adversarial attacks can broadly be grouped into two major classes -- \textit{white-box} and \textit{black-box} attacks. White-box attacks assume full knowledge of the network structure and parameters of the model to craft an attack. These methods typically find the perturbation by solving an optimization involving the gradients of the loss function with respect to the input $\nabla_\x J(\x, y, \theta)$, where $J$ is the classification loss, $(\x, y)$ is a data-point, and $\theta$ corresponds to the model parameters~\cite{goodfellow2014explaining, carlini2017towards, kurakin2016adversarial}. Black-box attacks on the other hand assume no access to the model, and construct attacks just by using the model's input-output response.

To protect the classifiers from adversarial attacks, one line of research focuses on removing the perturbation from input samples before feeding them to the classifier. MagNet~\cite{Meng2017MagNet} detects and reforms adversarial images using autoencoders. Defense-GAN~\cite{samangouei2018defense} uses a generative adversarial network to model the image manifold. Adversarial perturbations are removed by projecting the samples onto the learned manifold. A similar idea has been used in PixelDefend~\cite{song2018pixeldefend} where the input distribution is modeled using a PixelCNN, and adversarial perturbations are removed from the input samples using greedy decoding. These projection-based defense methods take advantage of the operations that leads to \emph{obfuscated gradients} which results in the inability to derive gradients to craft white-box attacks. In \cite{athalye2018obfuscated}, BPDA and reparameterization attacks are most related to this work. The basic idea of BPDA is to approximate the non-differentiable operations with surrogates (e.g., identity function) during backpropagation. Adversarial attacks can then be crafted on the target classifiers using the gradients derived from the surrogates. Reparameterization uses the change-of-variables trick to circumvent the operations resulting in vanishing gradients.

%% file: method.tex
\section{Method}\label{sec:method}

The generator network in a GAN model $G(\z):\R^d \rightarrow \R^{w \times h \times c}$ maps a latent code $\z \in \R^d$ to an image in $\R^{w \times h \times c}$. The inversion problem is to find the inverse mapping i.e., given a test sample $\x$, we are interested in finding a latent $\hat{\z}$ such that $G(\hat{\z}) \approx \x$. This problem is extremely hard as most generator networks used in practice are non-convex and non-bijective functions. One common approach to this problem~\cite{lipton2017precise, samangouei2018defense} involves solving the following optimization problem:
\begin{align}\label{eq:opt_simple}
    \min_\z \| G(\z) - \x \|_2
\end{align}
This optimization is extremely hard due to the non-convexity of the objective function. Solving this requires multiple random initialization of $\z$, carefully tuned learning rate and number of optimization steps for each dataset. Also, this optimization is extremely slow, and scales poorly with increasing complexity of the generator network and the input distribution. 

To fix these issues, we propose using an inverter network $I_{\theta_I}(\x): \R^{w \times h \times c} \rightarrow \R^d$ that maps a sample from the image space to the latent space as an initialization for $\z$ in \eqref{eq:opt_simple} . The goal of the inverter network is to approximately invert the generator network. We would like to emphasize that exact inversion is not possible as the generator function is non-bijective.


\begin{theorem}
Let $\{ G(\zi)\}_{i=1}^{N}$ represent the generated samples corresponding to a pre-trained generator function $G$, with the noise vectors $\zi \sim \cN(0, 1)$. Let $I(\cdot)$ represent an inverter function that is trained to achieve approximate inversion on the training set, i.e.,
\begin{align*}
    \| I(G(\zi)) - \zi \| < \epsilon, \quad \forall \zi, i \in [n].
\end{align*}
Let $L$ be the Lipschitz constant corresponding to the composite function $I \circ G(\cdot)$. Then, for $\epsilon' > \epsilon$, with probability $1 - o(1)$,
\begin{align*}
    \| I(G(\z)) - \z \| < \epsilon', ~~ \mbox{for} ~~ \z \sim \cN(0, 1).
\end{align*}
That is with high probability, the function $I(\cdot)$ approximately inverts the generator $G(\cdot)$. 
\end{theorem}
The above theorem states that under some smoothness conditions on the generator-encoder pair, if the encoder loss is bounded for every training sample, then the encoder approximately inverts the generator with high probability. Please refer to the supplementary material for the proof.


\subsection{Encoder training}
A natural way to train the encoder is to minimize the following loss function:
\begin{align} \label{eq:inv_basic}
    \min_{\theta_I} \mathbb{E}_{\x \in \X_{train}} \| \x - G(I_{\theta_I}(\x)) \|.
\end{align}
One issue with this objective arises from the non-surjective nature of the the generator network. There are many samples in the training set that cannot be represented by the generator network as it is not surjective. Enforcing the mean squared error (MSE) loss for such samples per Eq.~\eqref{eq:inv_basic} is not appropriate and leads to blurry reconstructions. Hence, we propose using the following losses for training the encoder.  

\paragraph{Approximate semantic consistency:}
To also make sure that our reconstructions are semantically consistent we add:
\begin{align}
    \cL_{semantic} = \Ex_{\z \sim p_z}\Big[\max(\norm{G(\z) - G(I(G(\z)))}, \eta) \Big]. \label{eq:semantic_loss}
\end{align}
The use of $L_2$ norm is not a good distance measure between two images, and minimizing the $L_2$ distance results in blurry reconstructions. Hence, we use hinge loss on $L_2$ norm in our formulation. The use of hinge loss in combination with adversarial loss (which we define later) searches for sharp reconstructions that are within $\eta$ $L_2$ norm ball of reconstruction error, instead of blurry reconstructions obtained by minimizing just the $L_2$ reconstruction error.

\paragraph{Latent code recovery:}
In addition to maintaining semantic consistency between the generated images and the inverted reconstructions, we recover the latent codes by making sure they are close to the sampled $\z$:
\begin{align}
    \cL_{latent} = \norm{\z - I(G(\z))}. \label{eq:latent_loss}
\end{align}

\paragraph{Inverted image distribution consistency:}
We want the images that are generated by the inversion $G(I(\X))$ to have the same distribution as the images that are generated from samples of the domain space of the generator $G(\z)$. Therefore, we add a discriminator at training time for which the the real samples are the generations $G(\z)$ and the fake samples are the inversions $G(I(G(\z))$:
\begin{align}
    \cL_{adv}(I, D) = \Ex_{\z \sim p_z}\big[\log(D_{\theta_D}(G(\z)))& - \log (1 - D_{\theta_D}(G(I(G(\z))))\big]. & \label{eq:dist_loss}
\end{align}
This adversarial loss is crucial to improving the quality of the reconstructions.

\subsection{Training} 
The encoder model is trained using a combination of the three loss terms introduced above. The objective function can be written as 
\begin{align*}
    \min_{\theta_I} \max_{\theta_D} \lambda_1 \cL_{adv}(I, D) + \lambda_2 \cL_{semantic}(I) + \lambda_3 \cL_{latent}(I).
\end{align*}
We set $\lambda_2=100, \lambda_1 = \lambda_3 = 1$ so that the semantic consistency gets enforced early on in the training, with the other two losses getting minimized gradually to correct for the distribution mismatch. We train in an iterative adversarial fashion to update the parameters of $I$ and $D$.

\subsection{Adversarial defenses}
The objective of adversarial defense mechanisms is to make the classifiers robust to any class of adversarial perturbation. In this work, we consider additive perturbations defined in Section~\ref{sec:background} -- the most common form of adversarial attacks used till date. However, our framework is general and can be extended to other forms of attacks as well. Given an adversarially perturbed image $\x^{adv} = \x + \bdelta$, projection-based defense mechanisms project the adversarial sample $\x^{adv}$ to the manifold representing the input dataset. In DefenseGAN~\cite{samangouei2018defense}, the image manifold is first modeled by training a GAN on the input dataset. The perturbed sample $\x^{adv}$ is then projected to the generative manifold by solving the optimization~\eqref{eq:inv_basic}.

In this work, we replace the inference step in DefenseGAN using our proposed InvGAN framework. Given a test sample $\x^{test}$, the approximate latent code is first obtained by passing through the inverter network $I(\cdot)$, which can then be used as an initialization to the optimization~\eqref{eq:inv_basic}:
\begin{align}
\z_0 &= I(\x^{adv}), \nonumber \\
\z_t &= \z_{t-1} - \alpha \nabla_{\z = \z_{t-1}}\norm{G(\z) - \x^{adv}}, \quad \forall t \in {1, 2, \hdots T}, \nonumber \\
\x^{proj} &= G(\z_T), \label{eq:adv_proj}
\end{align}
where $T$ denotes the number of inference iterations, and $\alpha$ is the learning rate. Let $f(\x): \R^{w \times h \times c} \rightarrow \R^n_c$ denote a classification network trained to predict a sample $\x$ into one of the $n_c$ classes. The projected sample $\x^{proj}$ computed using Eq.~\eqref{eq:adv_proj} is passed to the trained classification network as $f(\x^{proj})$ to make a prediction for the sample $\x$.

\paragraph{Attack detection: } In addition to robust classification, we provide a framework for detecting adversarially perturbed samples. Let the trained classification network $f$ be decomposed as $f(\x) = C \circ \Phi(\x)$ where $C$ denotes the last layer of the network, and $\Phi$ denotes all layers except the final layer. $\Phi(\x)$ gives a feature representation of the image $\x$. We define attack detection score $\cA(\x)$ as 
\begin{align}\label{eq:det_score}
    \cA(\x) = \| \Phi(\x) - \Phi(\x^{proj}) \|.
\end{align}
By choosing a threshold on the attack detection score, adversarially perturbed samples can be detected. To compute the detection score, we measure the projection distance in feature space and not in image space (i.e., $\| \x - \x^{proj} \|$) as mean squared error is not a good measure of distance between two images. An ablation study showing the advantage of using feature-based distance over image based distance is shown in Section \ref{sec:experiments}.

\subsection{Reparameterization white-box attack}

The inability to compute gradients for projection-based defense methods such as DefenseGAN makes it hard to craft white-box attacks. To attack these models, \cite{athalye2018obfuscated} constructs a reparameterization attack where the input samples $x$ are reparameterized as $x = h(z)$ for some differentiable function $z$ such that $g(h(z)) = h(z),~\forall z$. Then, the gradients of the classifier can be computed using $f(h(z))$ as $h(\cdot)$ is differentiable. Consider $h(\cdot)$ as the generator function, and $g(\cdot)$ as our encoder-generator pair. Through our encoder training, we almost achieve perfect inversion (Section~\ref{sec:experiments}) which satisfies $g(h(z)) = h(z),~\forall z$. Hence for a given test sample $\x$, we construct the adversarial attack using the gradient $\nabla_x f(G(I(\x)))$.

%% file: experiments.tex
\section{Experiments}\label{sec:experiments}
Our proposed approach is evaluated on four datasets: MNIST~\cite{mnist}, Fashion-MNIST~\cite{fmnist}, CIFAR-10~\cite{cifar}, and CelebA~\cite{celeba}. We pretrain GANs on all the datasets using the network architectures presented in~\cite{miyato2018spectral}. We use DCGAN architecture for MNIST and Fasion-MNIST, and GANs with residual blocks for CIFAR-10 and CelebA. The architecture of the inverter $I$ is the mirror image of the generator. The inverter network is trained for 100K iterations using the Adam optimizer with $\beta_1 = 0.5$, and $\beta_2 = 0.999$. Our implementation is based on Tensorflow~\cite{tensorflow} and Cleverhans~\cite{papernot2018cleverhans}.\footnote{The source code to reproduce all the experiments will be released after final decisions.}

\subsection{Inference}
In this experiment, we consider the task of inferring the latent representation $\z$ of an input image and reconstructing the input from the inferred $\z$. The closeness of the reconstructed image to the input illustrates the strength of the inversion scheme. The following quantitative metrics are considered for evaluation: (1) classification accuracy on a pre-trained classifier to assess whether the semantic information is retained in the reconstructed images, (2) Inception score (IS)~\cite{Salimans2016ImprovedTF}, (3) Fr\'{e}chet inception distance (FID)~\cite{Heusel2017GANsTB} of the reconstructed samples, (4) MSE between the input and reconstruction. The proposed InvGAN is compared with ALI~\cite{dumoulin2017ALI} and AGE~\cite{Ulyanov2018ALI} on the CIFAR-10 test set. We also consider the following baselines:
\begin{itemize}
    \item Direct optimization: $\z^* = \arg \min_\z \| G(\z) - \x \|_2$ is first solved by running gradient descent for 200 iterations. $G(\z^*)$ is then treated as the reconstructed image.
    \item InvGAN with $T=0$: The encoder-decoder scheme similar to ALI and AGE.
    \item InvGAN with $T=200$: The inference used at defense time.
\end{itemize}
For fair comparisons, in this experiment, we adopt the simple DCGAN architecture without residual blocks in~\cite{miyato2018spectral} on CIFAR-10. The results are presented in Table~\ref{table:inference} and Figure~\ref{fig:inversion_qualitative}. InvGAN with $T=0$ achieves the best IS, FID and accuracy than the competing methods, while direct optimization achieves the best MSE. However, lower MSE does not necessarily produce natural looking images (which is reflected in poorer inception and FID scores) since the MSE loss does not take semantic information into account. Also note that InvGAN only suffers slightly from running several steps of MSE updates. However, these optimization updates offer security against common end-to-end attacks as will be discussed in the following sections. 

\input{tables/inference.tex}

\subsection{Defense against adversarial attacks}
In this section, we compare InvGAN $(T=1000)$ with Defense-GAN~\cite{samangouei2018defense}, Cowboy~\cite{Santhanam2018Cowboy}, and PixelDefend~\cite{song2018pixeldefend} in defending against white-box attacks. In~\cite{samangouei2018defense,song2018pixeldefend}, end-to-end attacks are either not considered or not successful due to the difficulty involved in gradient computation. Following these works, we consider crafting FGSM~\cite{goodfellow2014explaining} and CW~\cite{carlini2017towards} attacks on the classifier and feeding the adversarial images to our pipeline. In addition, we also experiment on end-to-end attacks using reparameterization and BPDA~\cite{athalye2018obfuscated}. 

\paragraph{Robust classification:}
We use MNIST, Fashion-MNIST, CIFAR-10, and CelebA for evaluation. For the FGSM attack, we select $\epsilon=0.3$ on MNIST and Fashion-MNIST, and $\epsilon=8/255$ on CIFAR-10 and CelebA. For the CW attack, we set the binary search step to 6, learning rate to 0.2, and number of iterations to 100. Table~\ref{table:white_box} shows the classification accuracy. We observe that InvGAN achieves improved performance compared to DefenseGAN, and offers defense against the BPDA attack. In Figure~\ref{fig:attack_removal}, we visualize attack removal for DefenseGAN and InvGAN under BPDA attack. For each method, the first row represents reconstruction on clean images, the second row shows the attacked images, and the third row shows the reconstructed attacked images. We observe that DefenseGAN partially reconstructs the attacked images, while InvGAN successfully removes the perturbation produced by BPDA.


\input{tables/whitebox.tex}
\input{figures/figure_attack_remove.tex}
\input{tables/whitebox_detection.tex}

\paragraph{Attack detection:}
In Table~\ref{table:detection}, we compare the area under ROC curve (AUC) scores for attack detection between DefenseGAN and InvGAN on different adversarial attacks. Comparing Table~\ref{table:white_box} and Table~\ref{table:detection}, we observe that for large perturbation settings (e.g. FGSM with $\epsilon=0.3$) where the classification accuracy of InvGAN is low, attacks can be detected easily. On the other hand, for minute perturbations (e.g. FGSM with $\epsilon= 8/255$ and CW), InvGAN successfully removes perturbation and achieves high accuracy, but the attack becomes more challenging to detect. This suggests that our attack detection and robust classification scheme offers orthogonal benefits -- when one fails, the other succeeds.

\paragraph{Black-box attack:} In addition to whitebox attacks, we compare InvGAN with DefenseGAN under black-box attack~\cite{papernot2017practical} in Table~\ref{table:blackbox}. We use the same procedure descried in \cite{papernot2017practical} and \cite{samangouei2018defense}.

\subsection{Run time comparison}
Measuring the run time of defense mechanisms is challenging as it depends on the implementation. We propose using the \emph{effective number of inference iterations} as a measure of run time, defined as the product of number of random restarts and the number of iterations per run. For InvGAN, the number of random restarts is always $1$ as we initialize it from the encoder. We report the classification accuracy of DefenseGAN and InvGAN on Fashion-MNIST for clean images by varying the effective number of iterations. The result is presented in Figure~\ref{fig:running_time}. InvGAN has a significantly shorter run time and offers reconstructions with improved semantic consistency. 
\input{figures/figure_running_time.tex}

\subsection{Ablation study} 
In this section, we study the effect of disabling the adversarial loss. More specifically, we train the encoder model by setting $\eta=0$, $\lambda_1=0$, $\lambda_2=1$, and $\lambda_3=0$ . As can be seen from Table \ref{table:ablation_inference}, although this network achieves a lower MSE, the images are perceptually less realistic, and is reflected in lower classification accuracy.
\input{tables/ablation_inference.tex}

%% file: tables/inference.tex
\begin{table}[h!]
    \caption{Quantitative evaluation of inference on CIFAR-10 test set.}
    \label{table:inference}
    \centering
    \resizebox{0.7\textwidth}{!}{
    \begin{tabular}{@{}lcccc@{}}
    \toprule
         & MSE & IS & FID & Accuracy \\
    \midrule
    ALI     & $0.32 \pm 0.17$ & $6.12 \pm 0.15$ & $57.79$ & $0.35$      \\
    AGE     & $0.06 \pm 0.03$ & $6.43 \pm 0.15$ & $39.93$ & $0.43$      \\
    \midrule
    Direct Optimization & \textbf{0.03 $\pm$ 0.02} & $6.50 \pm 0.20$ & $40.18$ & $0.44$ \\
    InvGAN ($T=0$)  & $0.10 \pm 0.06$ & \textbf{7.72 $\pm$ 0.16} & \textbf{22.35} & \textbf{0.59} \\
    InvGAN ($T=200$)  & $0.08 \pm 0.04$ & $7.36 \pm 0.27$ & $23.91$ & \textbf{0.59} \\
    \bottomrule
    \end{tabular}
    }
\end{table}

%% file: tables/whitebox.tex
\begin{table}[h!]
    \caption{Classification accuracy under different white-box attacks. We directly report the performance presented by the authors for methods marked in $^\dagger$. BPDA attack introduces intense noise in experiments marked with *. Visualization of the attack is presented in the supplemental material. } \label{table:white_box}
    \centering
    \small
    \begin{tabular}{@{}lcccccccccc@{}}
    \toprule
    \multicolumn{1}{c}{} & \multicolumn{5}{c}{MNIST} & \multicolumn{5}{c}{Fashion-MNIST} \\
    \cmidrule(lr){2-6} \cmidrule(l){7-11}
        & Clean & FGSM & CW & RP & BPDA & Clean & FGSM & CW & RP & BPDA \\
    \midrule
    No Defense & 0.99 & 0.18 & 0.01 & 0.72 &  -   & 0.92 & 0.06 & 0.06 & 0.41 &     \\
    Cowboy$^\dagger$     &   -  & 0.78 &  - & - & - &  -   & 0.44 &  -   &   -  &  - \\
    DefenseGAN & 0.98 & \textbf{0.83} & 0.96 & 0.92 & 0.79 & 0.81 & \textbf{0.34} & 0.80 & \textbf{0.54} & 0.59 \\
    InvGAN     & \textbf{0.99} & 0.78 & \textbf{0.99} & \textbf{0.92} & \textbf{0.87} & \textbf{0.88} & 0.28 & \textbf{0.87} & 0.49 & \textbf{0.71} \\
    \midrule
    \multicolumn{1}{c}{} & \multicolumn{5}{c}{CIFAR-10} & \multicolumn{5}{c}{CelebA} \\
    \cmidrule(lr){2-6} \cmidrule(l){7-11}
         & Clean & FGSM & CW & RP & BPDA & Clean & FGSM & CW & RP & BPDA \\
    \midrule
    No Defense  & 0.95 & 0.16 & 0.02 & 0.84 &  -   & 0.97 & 0.05 & 0.04 & 0.96 & -    \\
    Cowboy$^\dagger$      &   -  & 0.53 & -  & - & - & - & - & - & -      & - \\
    PixelDefend$^\dagger$ & \textbf{0.85} & 0.46 & - & - & 0.09 & - & - & - & - & -     \\
    DefenseGAN  & 0.44 & 0.43 & 0.44 & 0.44 & n/a* & 0.89 & 0.88 & 0.88 & 0.89 & n/a* \\
    InvGAN      & 0.66 & \textbf{0.60} & \textbf{0.58} & \textbf{0.61} & n/a* & \textbf{0.93} & \textbf{0.90} & \textbf{0.92} & \textbf{0.90} & n/a* \\
    \bottomrule
    \end{tabular}
\end{table}

%% file: figures/figure_attack_remove.tex
\begin{figure}[!h]
    \centering
    \begin{minipage}{0.48\linewidth}
        \centering
        \includegraphics[width=0.9\textwidth]{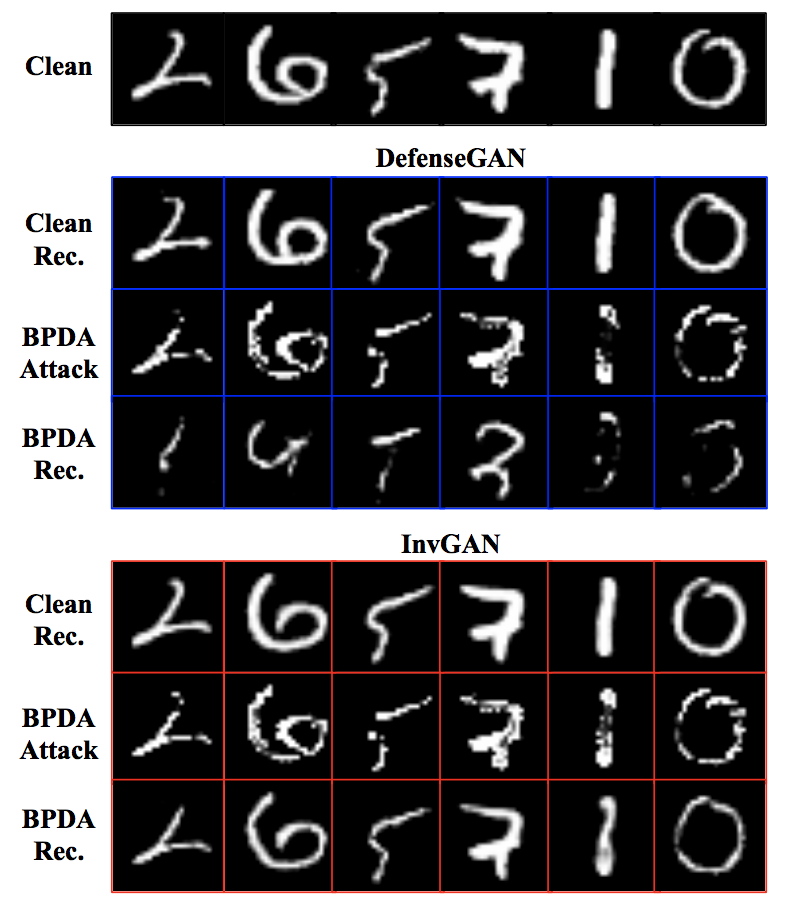}
        \caption{Illustration of adversarial attacks.}\label{fig:attack_removal}
    \end{minipage}
    \begin{minipage}{0.48\linewidth}
        \centering
        \includegraphics[width=\textwidth]{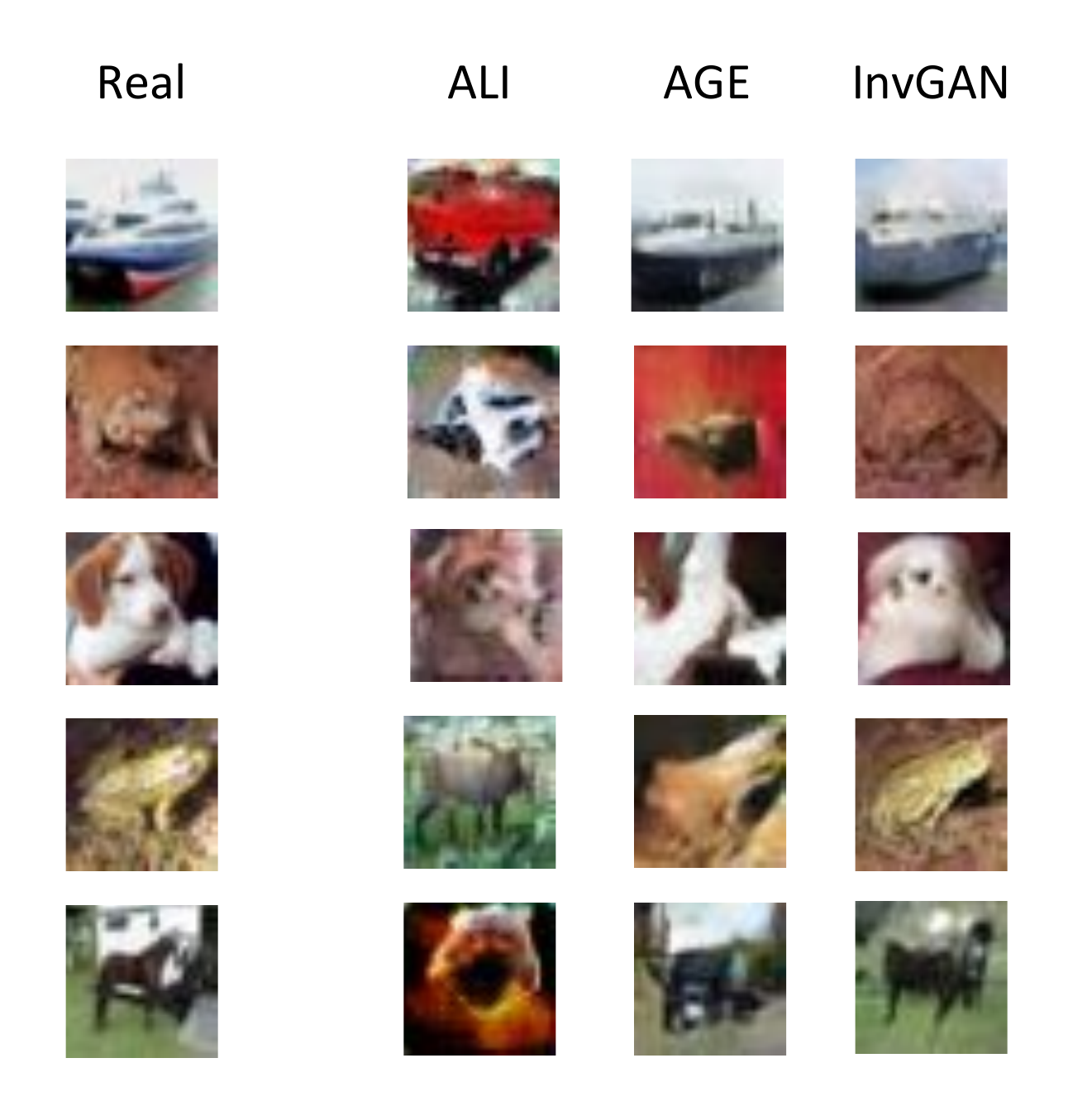}
        \caption{Qualitative inversion results.}
        \label{fig:inversion_qualitative}
    \end{minipage}
\end{figure}

%% file: tables/whitebox_detection.tex
\begin{table}[h!]  
    \centering 
    \caption{Attack detection performance (AUC) of DefenseGAN and InvGAN.} \label{table:detection}
    \resizebox{0.9\textwidth}{!}{
    \begin{tabular}{@{}l cccc cccc@{}}
    \toprule
        & \multicolumn{4}{c}{MNIST} & \multicolumn{4}{c}{Fashion-MNIST} \\
    \cmidrule(lr){2-5} \cmidrule(l){6-9}
        & FGSM & CW & RP & BPDA & FGSM & CW & RP & BPDA \\
    \midrule
    DefenseGAN & 0.9878 & 0.9661 & 0.9801 & 0.8706 & 0.9563 & 0.7752 & 0.9364 & 0.8999  \\
    InvGAN     & \textbf{0.9985} & \textbf{0.9880} & \textbf{0.9975} & \textbf{0.9210}  & \textbf{0.9932} & \textbf{0.8543} & \textbf{0.9845} & \textbf{0.9334} \\
    \midrule
    \multicolumn{1}{c}{} & \multicolumn{4}{c}{CIFAR-10} & \multicolumn{4}{c}{CelebA} \\
    \cmidrule(lr){2-5} \cmidrule(l){6-9}
        & FGSM & CW & RP & BPDA & FGSM & CW & RP & BPDA \\
    \midrule
    DefenseGAN  & 0.6524 & 0.6823 & 0.4537 & 0.5950 & 0.8753 & 0.7341 & 0.4488 & \textbf{0.7491} \\
    InvGAN      & \textbf{0.8132} & \textbf{0.8370} & \textbf{0.5469} & \textbf{0.7807} & \textbf{0.9042} & \textbf{0.7764} & \textbf{0.4781} & 0.7233 \\
    \bottomrule
    \end{tabular}
    }
\end{table}

%% file: figures/figure_running_time.tex

\begin{figure}[t]
    \centering
    \begin{minipage}{0.48\linewidth}
        \includegraphics[width=\textwidth]{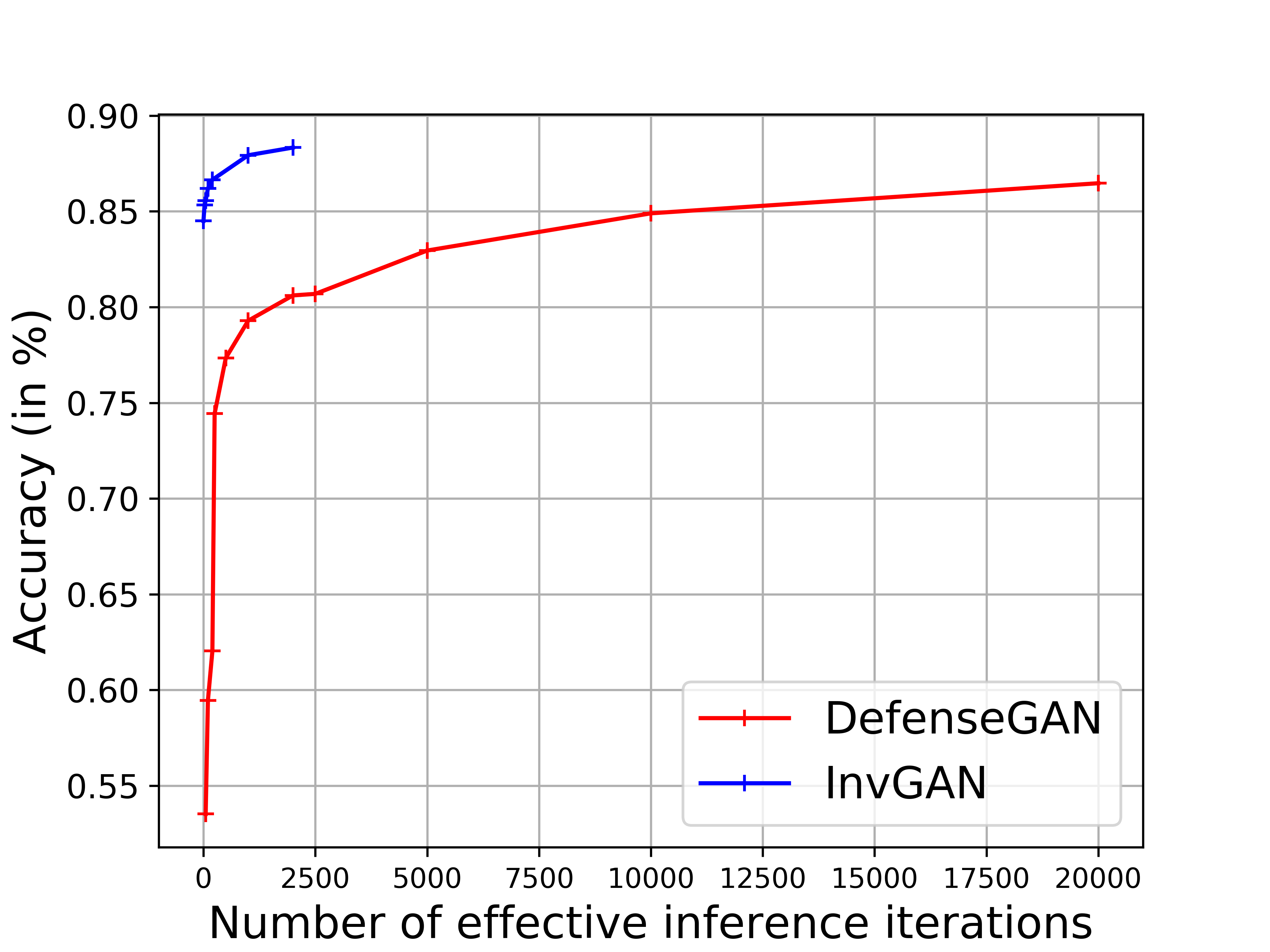}
        \captionof{figure}{Speed - accuracy trade-off curves.}
        \label{fig:running_time}
    \end{minipage}
    \begin{minipage}{0.48\linewidth}
        \centering
        \captionof{table}{Black-box attack performance comparison on MNIST and Fasion-MNIST.}\label{table:blackbox}
        \begin{tabular}{@{}lcc@{}}
        \toprule
                  & MNIST & FMNIST \\   
        \midrule
        No Defense & 0.6631 & 0.4988 \\
        DefenseGAN & 0.9398 & 0.6320  \\
        InvGAN & \textbf{0.9449} & \textbf{0.6596} \\
        \bottomrule
       \end{tabular}
    \end{minipage}
\end{figure}

%% file: tables/ablation_inference.tex
\begin{table}[h!]
    \caption{Analyzing the effect of adversarial loss in InvGAN.}
    \label{table:ablation_inference}
    \centering
    \begin{tabular}{@{}lcccc@{}}
    \toprule
         & MSE & IS & FID & Accuracy \\
    \midrule
    w/o $\cL_{adv}$ & $0.08 \pm 0.04$ & $7.54 \pm 0.17$ & $28.07$ & $0.603$ \\
    Ours & $0.10 \pm 0.06$ & \textbf{7.72 $\pm$ 0.16} & \textbf{19.85} & \textbf{0.625} \\
    \bottomrule
    \end{tabular}
\end{table}

%% file: conclusion.tex
\section{Conclusion}\label{sec:conclusion}

In this work, we introduce InvGAN -- a novel data-free and model-based inversion framework for solving the inference problem in GANs. Our approach involves training an encoder function capable of inverting the generator network back to the latent space. The encoder function is trained using a novel loss function that achieves superior inversion results compared to the contemporary methods performing inference. The usefulness of our inversion scheme is demonstrated in the problem of adversarial defenses, where our inversion scheme has been shown to achieve dramatic improvements in defense performance, run time and attack detection over DefenseGAN -- a projection-based defense mechanism using GANs.

%% file: supp.tex
\appendix

\section{Proof of Theorem 1}

\begin{theorem*}
Let $\{ G(\zi)\}_{i=1}^{N}$ represent the generated samples corresponding to a pre-trained generator function $G$, with the noise vectors $\zi \sim \cN(0, 1)$. Let $I(\cdot)$ represent an inverter function that is trained to achieve approximate inversion on the training set, i.e.,
\begin{align*}
    \| I(G(\zi)) - \zi \| < \epsilon, \quad \forall \zi, i \in [n].
\end{align*}
Let $L$ be the Lipschitz constant corresponding to the composite function $I \circ G(\cdot)$. Then, for $\epsilon' > \epsilon$, with probability $1 - o(1)$,
\begin{align*}
    \| I(G(\z)) - \z \| < \epsilon', ~~ \mbox{for} ~~ \z \sim \cN(0, 1).
\end{align*}
That is with high probability, the function $I(\cdot)$ approximately inverts the generator $G(\cdot)$. 
\end{theorem*}

\begin{proof}

The input samples $\{ G(\zi)\}_{i=1}^{N}$ correspond to the training data for the inverter network. For any latent $\z \sim \cN(0, 1)$,
\begin{align*}
    P(\|\z - \z_i\|<\epsilon) \geq 1 - e^{-\frac{d}{18} (\frac{\epsilon^2}{4d} - 1)^{2}} \quad \forall i \in [n].
\end{align*}
The above inequality follows from the concentration bound for $\chi^2$ distribution~\cite{wainwright_2019} since $\frac{1}{4}\|\z - \z_i\|^2$ follows a $\chi^{2}_d$ distribution with $d$ degrees of freedom, where $d$ is the noise dimension. Then, 
\begin{align}\label{eq: z_bound}
    P(\exists \zi, i \in [n] \mbox{ s.t } \|\z - \zi\| < \epsilon ) \geq 1 - e^{-\frac{nd}{18} (\frac{\epsilon^2}{4d} - 1)^{2}}.
\end{align}

Eq.~\eqref{eq: z_bound} says that there exists at least one $\zi$ concentrated close to $\z$. Now, consider $\| I(G(\z)) - \z \|$. This can be expanded as 
\begin{align}
    \| I(G(\z)) - \z \| &= \| (I(G(\z)) - I(G(\zi))) + (I(G(\zi)) - \zi) + (\zi - \z) \| \nonumber \\
        &\leq \| I(G(\z)) - I(G(\zi)) \| + \| I(G(\zi)) - \zi \| + \| \zi - \z \| \nonumber \\
        &\leq (L+1)\| \z - \zi \| + \| I(G(\zi)) - \zi \| \label{eq: tri_ineq} \\
        &\leq (L+1)\| \z - \zi \| + \epsilon, \quad \forall i \in [n] \nonumber \\
        &\leq \min_{i \in [n]}(L+1)\| \z - \zi \| + \epsilon. \nonumber
\end{align}

This follows from Triangle inequality and the assumption on training loss. Using~\eqref{eq: z_bound} and~\eqref{eq: tri_ineq} for bounding $\| I(G(\z)) - \z \|$, we obtain

\begin{align*}
    P(\| I(G(\z)) - \z \| \leq \epsilon') &\geq P\Big( [ \min_{i \in [n]}(L+1)\| \z - \zi \| ] < \epsilon' - \epsilon \Big) \\
            & = P(\exists i \text{  }  \|\z - \zi \| < \frac{\epsilon' - \epsilon}{L+1}) \\
            & \geq 1 - e^{-\frac{nd}{18} (\frac{(\epsilon' - \epsilon)^2}{4d(L+1)^2} - 1)^{2}}.
\end{align*}

That is with probability $1 - o(1)$, $\| I(G(\z)) - \z \| \leq \epsilon'$. Please note that we assumed that $\epsilon' > \epsilon$. This concludes the proof.
\end{proof}

\paragraph{Remark: }

In the above equation, $P(\| I(G(\z)) - \z \| \leq \epsilon')$ is taken with respect to $(\{ \z^{(i)} \}, \z)$. We define the event
\begin{align}
    E_n = \{ \omega: \|I_{\z^n(\omega)}(G(\z(\omega))) - \z(\omega) \| > \epsilon'\}.
\end{align}
Based on the bound, we have $P(E_n) < \exp(-nC)$, and thus $\sum_{n=1}^\infty P(E_n) < \infty$. Applying the Borel-Cantelli lemma, we conclude that with probability 1, the event $E_n$ happens for only finitely many $n$.


\section{Additional qualitative results}

\subsection{Image reconstruction}
We visualize the reconstruction results for DefenseGAN and InvGAN on CIFAR-10 and CelebA in Figures~\ref{fig:cifar_recon} and~\ref{fig:celeba_recon}. We can observe that with the initialization provided by the inverter, the reconstructed images are semantically more consistent compared to the direct optimization approach adopted by DefenseGAN.

\begin{figure}[!h]
   \centering
   \includegraphics[width=1.0\linewidth]{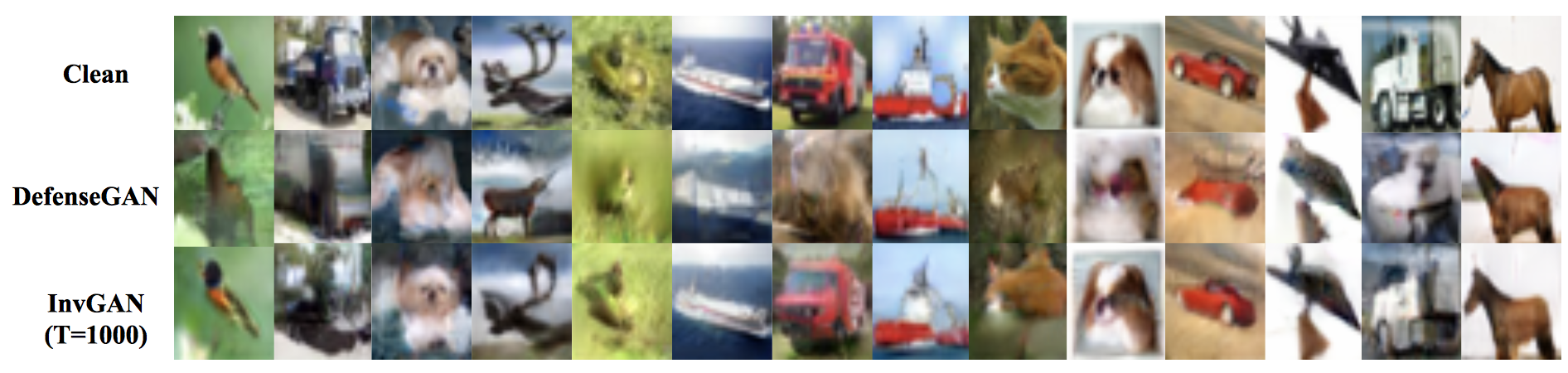}
   \caption{Reconstruction results on CIFAR-10. }\label{fig:cifar_recon}
\end{figure}
\begin{figure}[!h]
   \centering
   \includegraphics[width=1.0\linewidth]{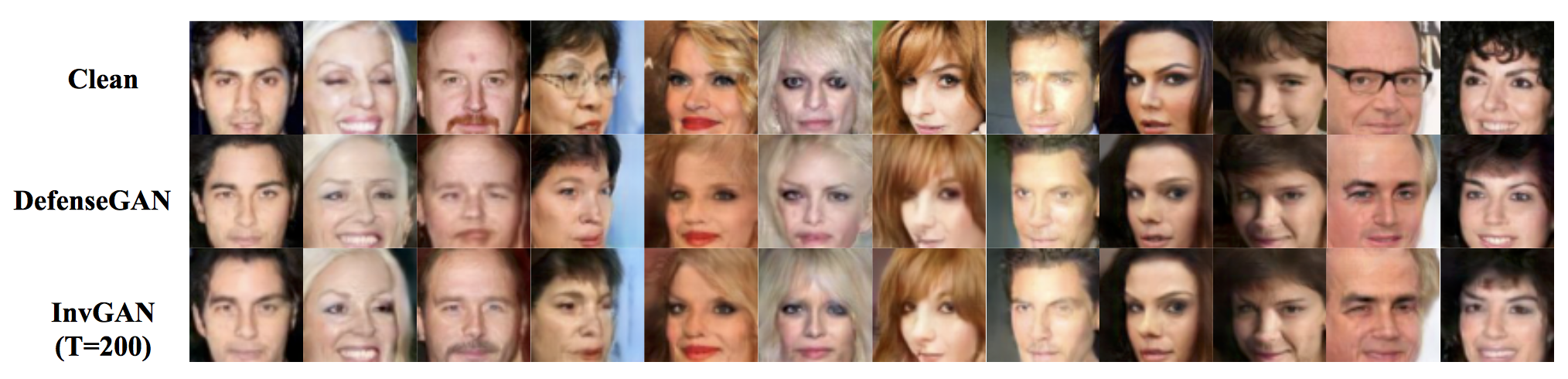}
   \caption{Reconstruction results on CelebA. }\label{fig:celeba_recon}
\end{figure}

\subsection{Attack removal}
We show sample results for InvGAN attack removal in Figure~\ref{fig:supp_mnist}. InvGAN effectively removes perturbations added by the attacks. 
\paragraph{Remark:} We observe that when the attack is mounted solely on the classifier (i.e. FGSM/CW on the classifier), the crafted perturbations are mainly focused on the non-stroke regions. That is, an attacker can easily manipulate the classification results of an unsecured classifier without modifying the semantic nature (i.e. the digit) of an image. In contrast, to mount attack on the end-to-end framework (i.e. BPDA on InvGAN + classifier), the perturbations have to be crafted to `erase' strokes.
\begin{figure}[!h]
   \centering
   \includegraphics[width=0.9\linewidth]{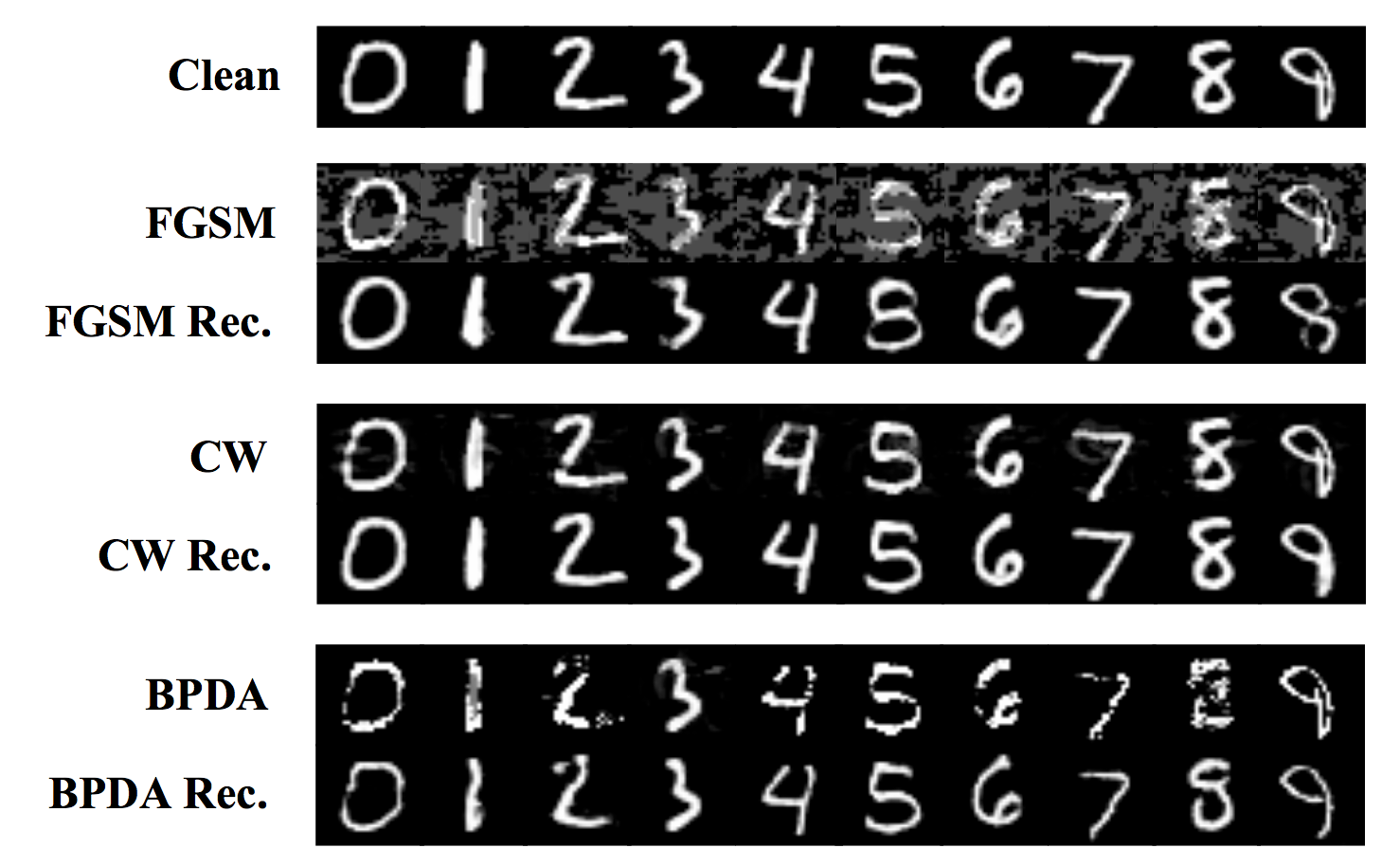}
   \caption{InvGAN attack removal results on MNIST. }\label{fig:supp_mnist}
\end{figure}

\subsection{BPDA attack}
We visualize the qualitative results of CW attack and BPDA attack on InvGAN model on CIFAR-10 dataset. We observe that while the perturbations produced by the CW attacks are imperceptible, BPDA attacks are very strong. In other words, it's hard to attack InvGAN with low perturbations. This is the reason why we did not include results on BPDA attack on CIFAR-10 and CelebA in Table 2 of the main paper.

\begin{figure}[h]
   \centering
   \includegraphics[width=1.0\linewidth]{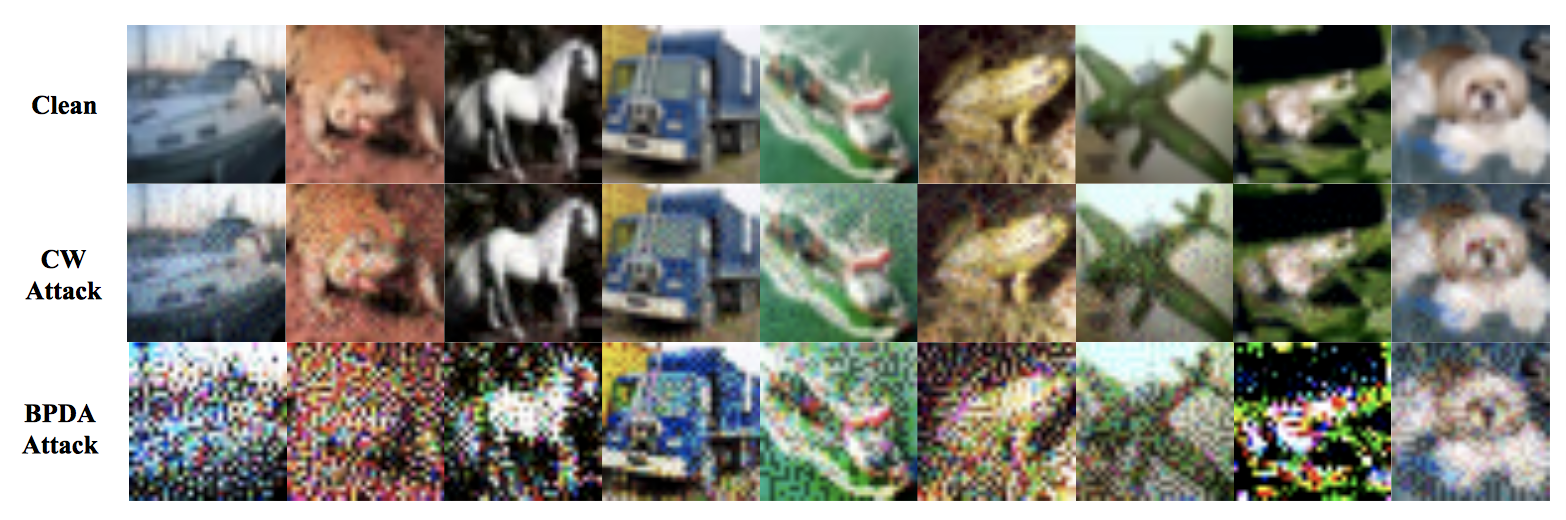}
   \caption{Sample images attacked by CW and BPDA. Clearly, BPDA mounts significant amount of perturbations on clean images. }\label{fig:supp_attack}
\end{figure}

